\RequirePackage{etex}
\documentclass{article}
\usepackage{amsmath, amsthm}
\usepackage{amsfonts, amssymb, mathtools, physics}
\usepackage{enumitem}
\usepackage{comment}
\usepackage{bbm}

\renewcommand{\Pr}{\mathop{\mathrm{Pr}}\nolimits}
\usepackage{bm}
\usepackage{float}
\usepackage{enumitem}
\newcommand{\ALG}{\mathrm{ALG}}
\newcommand{\OPT}{\mathrm{OPT}}
\renewcommand{\epsilon}{\varepsilon}
\usepackage{cleveref,autonum}
\usepackage[whole]{bxcjkjatype}
\usepackage{natbib}
\usepackage{fullpage}
\usepackage{subcaption}

\newtheorem{thm}{Theorem}

\newtheorem{dfn}{Definition}
\newtheorem{lemma}{Lemma}

\usepackage{authblk}
\title{Analyzing the effect of prediction accuracy on the distributionally-robust competitive ratio}
\author[1]{Toru Yoshinaga\thanks{yoshinaga-toru106@g.ecc.u-tokyo.ac.jp}}
\author[1]{Yasushi Kawase\thanks{kawase@mist.i.u-tokyo.ac.jp}}
\affil{The University of Tokyo, Japan}
\date{}

\begin{document}

\maketitle
\begin{abstract}
The field of algorithms with predictions aims to improve algorithm performance by integrating machine learning predictions into algorithm design. 
A central question in this area is how predictions can improve performance, and a key aspect of this analysis is the role of prediction accuracy.
In this context, prediction accuracy is defined as a guaranteed probability that an instance drawn from the distribution belongs to the predicted set.
As a performance measure that incorporates prediction accuracy, we focus on the distributionally-robust competitive ratio (DRCR), introduced by Sun et al.~(ICML 2024).
The DRCR is defined as the expected ratio between the algorithm's cost and the optimal cost, where the expectation is taken over the worst-case instance distribution that satisfies the given prediction and accuracy requirement.
A known structural property is that, for any fixed algorithm, the DRCR decreases linearly as prediction accuracy increases. 
Building on this result, we establish that the optimal DRCR value (i.e., the infimum over all algorithms) is a monotone and concave function of prediction accuracy.
We further generalize the DRCR framework to a multiple-prediction setting and show that monotonicity and concavity are preserved in this setting.
Finally, we apply our results to the ski rental problem, a benchmark problem in online optimization, 
to identify the conditions on prediction accuracies required for the optimal DRCR to attain a target value.
Moreover, we provide a method for computing the critical accuracy, defined as the minimum accuracy required for the optimal DRCR to strictly improve upon the performance attainable without any accuracy guarantee.
\end{abstract}

\section{Introduction}
Online algorithms are typically designed to guarantee performance under worst-case analysis.
A standard measure for evaluating the performance of an algorithm is the \emph{competitive ratio}, defined as the ratio of the cost incurred by an online algorithm to the optimal offline cost in the worst-case instance.
However, an exclusive focus on ensuring a good competitive ratio often becomes overly conservative, resulting in poor performance in practice.
This motivates leveraging predictions to improve performance while preserving worst-case guarantees.
A recent line of work develops this idea through the framework of \emph{algorithms with predictions}~\citep{PSK18, MV20, LV21}.
In this framework, the algorithm receives additional information about the problem instance, often derived from machine learning predictions.
The aim is to leverage predictions to achieve improved performance, while maintaining worst-case guarantees.

In most previous work on algorithms with predictions, algorithms have been evaluated by two types of performance measures: \emph{consistency}, which is the competitive ratio on instances that satisfy the prediction, and \emph{robustness}, which is the competitive ratio over all instances.
Formally, these two measures are defined as follows.
Let $\mathcal{I}$ be the set of all instances for a minimization problem.
A prediction induces an instance subset $\Theta\subseteq\mathcal{I}$ that is consistent with it.
We refer to $\Theta$ as the \emph{consistent set}.
Denote the expected cost incurred by an algorithm $A$ on an instance $I\in\mathcal{I}$ by $\mathrm{ALG}(I)$, and the offline optimal cost by $\mathrm{OPT}(I)$.
The consistency and robustness of algorithm $A$ are defined respectively as
\begin{align}
c(A)=\sup_{I\in\Theta}\frac{\mathrm{ALG}(I)}{\mathrm{OPT}(I)}
\quad\text{and}\quad
r(A)=\sup_{I\in\mathcal{I}}\frac{\mathrm{ALG}(I)}{\mathrm{OPT}(I)}.
\end{align}
These measures capture only two extreme scenarios, namely when the prediction is completely correct and when it is completely unreliable, and thus ignore information about prediction accuracy.
In contrast, machine-learning predictions are often associated with quantified accuracy guarantees.
From a risk management perspective in algorithm design, prediction accuracy plays a crucial role: high accuracy justifies greater reliance on predictions and more aggressive designs, whereas low accuracy necessitates conservative designs based on worst-case guarantees.
Accordingly, in real-world applications, it is necessary to design algorithms that explicitly leverage the available prediction accuracy.

In recent years, \citet{SHCHWB24} proposed a new performance measure that incorporates prediction accuracy, called the \emph{distributionally-robust competitive ratio (DRCR)}.
In their model, an instance is drawn from a distribution within the set of distributions specified by a consistent set $\Theta$ and a prediction accuracy $1-\delta$.
The instance distribution set $\mathcal{D}_{\Theta,\delta}$ includes all distributions on $\mathcal{I}$ such that a randomly drawn instance $I$ belongs to $\Theta$ with probability at least $1-\delta$.
With this notation, DRCR is defined as follows:
\begin{dfn}{(Distributionally-robust competitive ratio~\citep{SHCHWB24})}\label{dfn:DRCR} For problem $\mathcal{I}$, given a consistent set $\Theta$ and prediction accuracy $1-\delta$, the DRCR of algorithm $A$ is defined as
    $$\mathrm{DRCR}_{\Theta,\delta}(A)=\sup_{d\in\mathcal{D}_{\Theta,\delta}}\mathbb{E}_{I\sim d}\left[\frac{\mathrm{ALG}(I)}{\mathrm{OPT}(I)}\right].$$
\end{dfn}

Our primary objective in this study is to establish a theoretical foundation for the prediction accuracy required for an online algorithm to achieve a desired performance level in the DRCR.
These benchmarks, in turn, will inform the sample size or the choice of machine learning methods for constructing predictions when designing algorithms that leverage predictions.
To this end, we investigate the relationship between prediction accuracy and the optimal DRCR, the infimum of the DRCR over all online algorithms under the given prediction and accuracy.
Intuitively, higher prediction accuracy should yield better performance.
We can verify this intuition from the convex-combination representation of the DRCR given by \citet{SHCHWB24}.
\begin{thm}[\citep{SHCHWB24}]\label{thm:DRCRconvexOnelayer}
Fix a consistent set $\Theta$. The DRCR of algorithm $A$ can be expressed as a convex combination of its consistency and robustness, parameterized by the accuracy. That is,
    $$\mathrm{DRCR}_{\Theta,\delta}(A)=(1-\delta)\cdot c(A)+\delta\cdot r(A).$$
\end{thm}
However, this result alone is not sufficient to characterize the relationship between prediction accuracy and the optimal DRCR.
Changes in prediction accuracy can significantly alter the reliance on predictions in algorithm design and can also substantially change the construction of the algorithm itself.
Therefore, the relationship between prediction accuracy and the optimal DRCR is more complex, and a detailed analysis of this relationship is needed to examine the conditions under which predictions are desirable for algorithm design.

Alongside the structural analysis of the optimal DRCR, we extend the DRCR framework to handle multiple predictions, assuming natural scenarios in which the decision maker has access to multiple predictive sources.
In this setting, we assume a \emph{hierarchical} structure over multiple predictions and their accuracies.
At a high level, the hierarchical structure implies that the predicted instance sets and their associated accuracy guarantees are nested, and that the accuracy guarantee becomes stronger as we allow a larger set of instances.
For example, suppose we want to predict the number of times a person will use the gym in a month.
In this case, suppose we are given the following two predictions: the probability that the person visits the gym $3$--$5$ times in a month is $50\%$, and the probability that the number of visits is in the (wider) range $2$--$10$ is $80\%$.
These two predictions are hierarchical.
Under this hierarchical structure, the structural properties of the optimal DRCR are preserved in the multiple-prediction setting.

\subsection{Our results}
In Section~\ref{sec:DRCRstructures}, we present several structural properties of the DRCR with respect to prediction accuracy.
We first give a formal definition of the optimal DRCR and show that, in the single-prediction setting, it is a monotone and concave function of the prediction accuracy.
This property follows from the convex-combination representation of the DRCR given in Theorem~\ref{thm:DRCRconvexOnelayer}.
Next, we show that in the presence of hierarchical predictions, the DRCR can also be represented as a convex-combination of consistencies and robustness, parameterized by the prediction accuracies.
From these results, we immediately obtain that the optimal DRCR in the hierarchical prediction setting also preserves its monotonicity and concavity with respect to the prediction accuracies.
These structural properties show that if there exist a pair of a consistent set and its accuracy that contains enough information to improve the optimal DRCR, then any such pair with higher accuracy provides even more effective information for improving the optimal DRCR.

In Section~\ref{sec:skirental}, we develop a method for identifying the conditions on predictions required to improve the optimal DRCR for the ski rental problem, a representative benchmark in online optimization.
In our setting, the last day of skiing is predicted via hierarchical prediction intervals. First, we provide a linear programming (LP) formulation that computes an optimal randomized algorithm for this problem.
Then, leveraging this LP, we describe the conditions on the prediction accuracies required for the optimal DRCR to reach the desired value in the form of a system of linear inequalities.
Finally, restricting to the single-interval setting, we present a method for computing the minimum prediction accuracy at which the optimal DRCR is strictly better than the optimal robustness value (i.e., the optimal DRCR without any accuracy guarantees).

\subsection{Related work}
Classical competitive analysis evaluates the performance of online algorithms against worst-case instances, but this often leads to overly pessimistic guarantees.
To address this issue, researchers have explored models that reflect realistic situations, for example by relaxing the assumptions on the input model or by allowing some prior information about future inputs.
This line of research is surveyed in the textbook by \citet{T21}.
In particular, the DRCR framework can be viewed as belonging to the line of work on algorithms with predictions, as it assumes that partial information about future inputs is available via machine-learned predictions.
At the same time, it is closely related to the \emph{diffused adversary model} introduced by \citet{KP00}, where the adversary selects the worst-case distribution from a restricted family known to the decision maker.
The DRCR framework can be seen as a refinement of this perspective, with the assumption of partial knowledge about the distribution class being naturally justified by such predictions.

The ski rental problem is a fundamental problem in online optimization, and numerous variants of it have been studied.
In the classical setting, \citet{KMRS88,KMMO94} proposed deterministic and randomized algorithms that achieve the best possible competitive ratios, respectively.
Within the line of work on algorithms with predictions, many variants of the ski rental problem have been studied~\citep{ACEPS21, ADJKR24, AGP20, B20, BD22, BMS20, BP23, DKTVZ21, DNS23, EK25, GP19, IBB24, MLHLSL23, PSK18, SHCHWB24, SLA23, SLLA23, SR21, SVW25, WLW20, WZ20, ZTCD24}.
Among these works, \citet{SVW25} is most closely related to our study (and to \citet{SHCHWB24}), as it adopts a stochastic input model and explicitly quantifies prediction accuracy.
They incorporate a calibrated probabilistic predictor for whether the skiing duration exceeds the purchasing threshold into the algorithm design and show that the competitive ratio smoothly improves as the predictor becomes more accurate.

\section{Preliminaries}
In this section, we introduce notation for analyzing the DRCR with hierarchical predictions.
Throughout the paper, we use the notation $[n]=\{1,2,\dots,n\}$, where $n$ is a positive integer.
Let $\mathcal{I}$ denote the instance set of an online minimization problem.
An instance $I\in\mathcal{I}$ is randomly drawn from an unknown distribution $d$ over $\mathcal{I}$.
As predictive information, we are given a set of $n$ predictions $(\Theta_i,\delta_i)_{i\in[n]}$.
Each pair $(\Theta_i,\delta_i)$ consists of a consistent set $\Theta_i\subseteq\mathcal{I}$ and a corresponding accuracy parameter $\delta_i\in[0,1]$.
For each $i\in[n]$, the prediction $(\Theta_i,\delta_i)$ guarantees that the unknown distribution $d$ satisfies
$$\Pr_{I\sim d}[I\in\Theta_i]\ge1-\delta_i.$$

We sometimes want to treat only the consistent sets or only the accuracy parameters collectively.
Thus, we denote them by $\bm\Theta\coloneqq(\Theta_i)_{i\in[n]}$ and $\bm\delta\coloneqq(\delta_i)_{i\in[n]}$, respectively.
Using this notation, a collection of $n$ predictions can be simply written as $(\bm\Theta,\bm\delta)$, and we refer to it as a \emph{prediction profile}.
Let $\mathcal{D}_{\bm{\Theta},\bm{\delta}}$ denote the set of all distributions that satisfy the above conditions.
That is,
$$\mathcal{D}_{\bm{\Theta},\bm{\delta}}=\left\{d\mid \Pr_{I\sim d}[I\in\Theta_{i}]\ge1-\delta_i\;\text{for all }i\in[n]\right\}.$$
With this notation, the DRCR for algorithm $A$ with prediction profile $(\bm{\Theta},\bm{\delta})$ is defined as
$$\mathrm{DRCR}_{\bm{\Theta},\bm{\delta}}(A)=\sup_{d\in\mathcal{D}_{\bm{\Theta},\bm{\delta}}}\mathbb{E}_{I\sim d}\left[\frac{\mathrm{ALG}(I)}{\mathrm{OPT}(I)}\right].$$

In our setting, we assign a separate consistency to each prediction.
For each consistent set $\Theta_i$, we define the consistency $c_i$ as
$$c_{i}(A)=\sup_{I\in\Theta_{i}}\frac{\mathrm{ALG}(I)}{\mathrm{OPT}(I)}.$$
Also, as described in Introduction, robustness $r(A)$ is defined as the worst-case performance over all instances.
Here, we set $\Theta_{n+1}\coloneqq\mathcal{I}$ and denote $r(A)=c_{n+1}(A)$, for convenience.
By definition, we have $c_i(A)\ge1$ for all $i\in[n]$ and $r(A)\ge1$.

We say that the prediction profile $(\bm{\Theta},\bm{\delta})$ has a \emph{hierarchical} structure if
$$\Theta_1\subseteq\Theta_2\subseteq\cdots\subseteq\Theta_n\quad\text{and}\quad\delta_1\ge\delta_2\ge\cdots\ge\delta_{n}.$$
If some consecutive predictions have identical consistent sets and accuracy parameters, namely,
$\Theta_i=\Theta_{i+1}$ and $\delta_i=\delta_{i+1}$,
we can merge them into a single prediction without affecting any of our performance measures.
 
In addition to the previous definition $\Theta_{n+1}\coloneqq\mathcal{I}$, we further set $\Theta_0=\emptyset$.
Accordingly, set $\delta_0=1$ (since $\Pr_{I\sim d}[I\in\emptyset]=0$ for all $d$) and $\delta_{n+1}\coloneqq0$ (since $\Pr_{I\sim d}[I\in\mathcal{I}]=1$ for all $d$).
Adding this notation does not break the hierarchical structure.
By the inclusion chain $\Theta_1\subseteq\Theta_2\subseteq\cdots\subseteq\Theta_n\subseteq\Theta_{n+1}=\mathcal{I}$ given by the hierarchical structure, the consistencies and the robustness satisfy $1\le c_1(A)\le c_2(A)\le\dots\le c_n(A)\le r(A)$.
We make frequent use of this property in the proofs.

\medskip

To analyze the effect of prediction accuracy on the optimal DRCR, we study the ski rental problem, a standard benchmark in online optimization.
In this minimization problem, information about when a skier stops skiing is critical for minimizing the total cost.
We therefore leverage hierarchical predictions to estimate the last day of skiing.
As a natural setting for hierarchical predictions, we consider nested prediction intervals.
Formally, this problem is defined as follows.
\begin{dfn}[Ski rental problem with prediction intervals]
    A beginner skier plans to continue skiing until they get bored, and we denote this unknown day by $\tau$. The skier does not have their skis, so they must either rent them for $\$1$ per day or buy them for $\$B$. Here $B$ is an integer greater than $1$. On the first day, the skier is notified of the purchase cost $B$ and a prediction profile $(\bm{\Theta},\bm{\delta})$. Let the instance set $\mathcal{I}$ be the set of positive integers, and assume that the last skiing day $\tau$ follows an unknown probability distribution on $\mathcal{I}$ that satisfies the prediction conditions. Each consistent set $\Theta_i$ is an interval $[\ell_i,u_i]$ where $\ell_i$ is a positive integer and $u_i$ is either a positive integer or $+\infty$. Each accuracy guarantees that $\Pr[\tau\in[\ell_i,u_i]]\ge1-\delta_i$. Here, we assume that $\ell_n\le\cdots\le\ell_1\le u_1\le\cdots\le u_n$. This ensures that the prediction profile has a hierarchical structure. Our goal is to design an algorithm that minimizes the total cost incurred.
\end{dfn}

\section{Structures of the DRCR}\label{sec:DRCRstructures}
In this section, we analyze the structure of the DRCR by establishing its convex-combination property and the concavity of the optimal DRCR with respect to the prediction accuracies.

\subsection{Concavity of the optimal DRCR}\label{sec:concave}
We demonstrate concavity and monotonicity of the optimal DRCR with respect to the prediction accuracy parameter(s) by exploiting its convex-combination representation.

We begin with the single prediction case and then lift the argument to hierarchical predictions.
Fix a prediction $\bm{\Theta}=\{\Theta\}$ and treat the prediction accuracy $\bm{\delta}=\{\delta\}$ as a variable.
As shown in Theorem~\ref{thm:DRCRconvexOnelayer}, the DRCR with a single prediction is a linear function of $\delta$ for every algorithm $A$.
Define the optimal DRCR as a function of $\delta$ as follows:
$$\mathrm{DRCR}^*(\delta)\coloneqq\inf_{A}\mathrm{DRCR}_{\bm{\Theta},\bm{\delta}}(A)=\inf_{A}\{(1-\delta)\cdot c(A)+\delta\cdot r(A)\}.$$
By a well-known result in convex analysis, $\mathrm{DRCR}^*(\delta)$ is the pointwise minimum of linear functions and is therefore concave (see, e.g., \cite{B04}).
Furthermore, since $c(A)\le r(A)$ holds for every algorithm, the coefficient of $\delta$ is non-negative.
Hence, $\mathrm{DRCR}^*$ is a non-decreasing function of $\delta$.

\begin{thm}\label{thm:DRCRConcave}
     Given a fixed consistent set $\Theta$, the optimal DRCR is concave and monotone non-decreasing with respect to $\delta$.
\end{thm}

\subsection{Convex combination representation for hierarchical predictions}
\citet{SHCHWB24} demonstrated that, when a single prediction is given, the DRCR can be expressed as a convex combination of consistency and robustness (see \Cref{thm:DRCRconvexOnelayer}).
We extend this result by proving that a similar representation is also possible when hierarchical predictions are given.
The key tool in our proof is Abel's lemma on summation by parts, stated below.

\begin{lemma}[\cite{A26}]\label{lem:abel}
    For any sequences $\{a_i\}_{i=0}^{n+1}$ and $\{b_i\}_{i=0}^{n+1}$, 
$$\sum_{i=1}^{n+1} a_i b_i
=\sum_{i=1}^{n+1}a_i\cdot b_{n+1}-\sum_{i=1}^{n}\left(\sum_{j=1}^{i}a_j\right)(b_{i+1}-b_i).$$
\end{lemma}

Using this lemma, we obtain the following convex-combination representation.
\begin{thm}
    \label{thm:DRCRConvexCombination}
    Given a hierarchical prediction profile $(\bm{\Theta},\bm{\delta})$, DRCR of algorithm $A$ can be expressed as
    \begin{align}
        \mathrm{DRCR}_{\bm{\Theta},\bm{\delta}}(A)&\coloneqq\sup_{d\in\mathcal{D}_{\bm{\Theta},\bm{\delta}}}\mathbb{E}_{I\sim d}\left[\frac{\mathrm{ALG}(I)}{\mathrm{OPT}(I)}\right]
        =\sum_{i=1}^n(\delta_{i-1}-\delta_i)\cdot c_i(A)+\delta_n\cdot r(A).
    \end{align}
\end{thm}
\begin{proof}
    For each $i\in[n+1]$, let $\mathbbm{1}_{i}$ denote the indicator function of the subset $\Theta_i\setminus\Theta_{i-1}\subseteq\mathcal{I}$. 
    Additionally, define $p^{(d)}_i=\mathbb{E}_{I\sim d}[\mathbbm{1}_i(I)]$ as the probability that an instance $I$ drawn from an unknown distribution $d$ belongs to $\Theta_i\setminus\Theta_{i-1}$.
    Then, for each $i\in[n+1]$, the sum $\sum_{j=1}^i p_j^{(d)}$ is the probability that an instance $I$ drawn from $d$ belongs to $\Theta_i$, and it is at least $1-\delta_i$ if $d\in \mathcal{D}_{\bm{\Theta},\bm{\delta}}$.

    We first provide the upper bound on the DRCR.
    We can express the DRCR as follows:
\begin{align}
    \mathrm{DRCR}_{\bm{\Theta},\bm{\delta}}(A)&=\sup_{d\in\mathcal{D}_{\bm{\Theta},\bm{\delta}}}\mathbb{E}_{I\sim d}\left[\frac{\mathrm{ALG}(I)}{\mathrm{OPT}(I)}\right]\\
    &=\sup_{d\in\mathcal{D}_{\bm{\Theta},\bm{\delta}}}\mathbb{E}_{I\sim d}\Bigg[\sum_{i=1}^{n+1}\mathbbm{1}_{i}(I)\cdot\frac{\mathrm{ALG}(I)}{\mathrm{OPT}(I)}\Bigg]\\
    &\le\sup_{d\in\mathcal{D}_{\bm{\Theta},\bm{\delta}}}\mathbb{E}_{I\sim d}\Bigg[\sum_{i=1}^{n+1}\mathbbm{1}_{i}(I)\cdot c_i(A)\Bigg] \\
    &=\sup_{d\in\mathcal{D}_{\bm{\Theta},\bm{\delta}}}\Bigg[\sum_{i=1}^{n+1}\mathbb{E}_{I\sim d}[\mathbbm{1}_i(I)]\cdot c_i(A)\Bigg]\\
    &=\sup_{d\in\mathcal{D}_{\bm{\Theta},\bm{\delta}}}\Bigg[\sum_{i=1}^{n+1} p^{(d)}_i\cdot c_i(A)\Bigg].
\end{align}
By applying \Cref{lem:abel}, we have
\begin{align}
\sup_{d\in\mathcal{D}_{\bm{\Theta},\bm{\delta}}}\Bigg[\sum_{i=1}^{n+1} p^{(d)}_i\cdot c_i(A)\Bigg]
    &=\sup_{d\in\mathcal{D}_{\bm{\Theta},\bm{\delta}}}\Bigg[\sum_{j=1}^{n+1} p^{(d)}_j c_{n+1}(A)-\sum_{i=1}^n\left(\sum_{j=1}^ip^{(d)}_j\right)(c_{i+1}(A)-c_i(A))\Bigg]\\
    &\le c_{n+1}(A)-\sum_{i=1}^n(1-\delta_i)\cdot(c_{i+1}(A)-c_i(A))\\
    &= c_{n+1}(A)-\sum_{i=2}^{n+1}(1-\delta_{i-1})\cdot c_{i}(A)+\sum_{i=1}^n(1-\delta_i)\cdot c_i(A)\\
    &= \sum_{i=1}^n (\delta_{i-1}-\delta_i)\cdot c_i(A)+\delta_n\cdot r(A),
\end{align}
where the inequality follows from the facts that $\sum_{j=1}^ip^{(d)}_j\ge1-\delta_i$ and $c_i\le c_{i+1}$.

Next, we present the lower bound on the DRCR.
Let $\epsilon$ be a positive real. 
For each $i\in[n+1]$, let $I^{(\epsilon)}_i\in\Theta_i$ be an instance such that 
\begin{align}
\ALG(I^{(\epsilon)}_i)/\OPT(I^{(\epsilon)}_i)\ge (1-\epsilon)\cdot c_i(A).
\end{align}
Such an $I^{(\epsilon)}_i$ always exists, since $c_i(A)=\sup_{I\in\Theta_i}\ALG(I)/\OPT(I)$.
Let $d^{(\epsilon)}$ be the distribution that selects $I^{(\epsilon)}$ with probability $\delta_{i-1}-\delta_i$ for each $i\in[n+1]$.
It is straightforward to verify that $d^{(\epsilon)}\in \mathcal{D}_{\bm{\Theta},\bm{\delta}}$.
Then, we have
\begin{align}
    \mathrm{DRCR}_{\bm{\Theta},\bm{\delta}}(A)
    &=\sup_{d\in\mathcal{D}_{\bm{\Theta},\bm{\delta}}}\mathbb{E}_{I\sim d}\left[\frac{\mathrm{ALG}(I)}{\mathrm{OPT}(I)}\right]\\
    &\ge\mathbb{E}_{I\sim d^{(\epsilon)}}\left[\frac{\mathrm{ALG}(I)}{\mathrm{OPT}(I)}\right]\\
    &=\sum_{i=1}^{n+1}\frac{\mathrm{ALG}(I^{(\epsilon)}_{i})}{\mathrm{OPT}(I^{(\epsilon)}_{i})}\cdot(\delta_{i-1}-\delta_i)\\
    &\ge \sum_{i=1}^{n+1}(1-\epsilon)\cdot c_i(A)\cdot(\delta_{i-1}-\delta_i)\\
    &=(1-\epsilon)\cdot\left(\sum_{i=1}^n (\delta_{i-1}-\delta_i)\cdot c_i(A)+\delta_n\cdot r(A)\right).
\end{align}
Taking the limit $\epsilon\to+0$ yields the desired lower bound.
\end{proof}
From this result, we can conclude that the optimal DRCR is non-increasing and concave in $\bm{\delta}$.
This is because for any fixed algorithm $A$, the DRCR is linear in $\delta_i$ and the coefficient of $\delta_i$ is $c_{i-1}(A)-c_i(A)$, which is non-negative. 

\section{Application: Ski rental problem}\label{sec:skirental}
In this section, we apply the structural properties obtained in \Cref{sec:DRCRstructures} to the ski rental problem with hierarchical prediction intervals.
Recall that, in this problem, the rental cost is $1$ per day and the purchase cost is an integer $B$.
The set of instances $\mathcal{I}$ consists of all positive integers.
The prediction profile $(\bm{\Theta},\bm{\delta})$ specifies intervals $\Theta_i=[\ell_i,u_i]$ for each $i\in[n]$, and these intervals satisfy $\ell_n\le\dots\le\ell_1\le u_1\le\dots\le u_n$.

In \Cref{sec:SkiRentalLPform}, we present an LP-based method to compute an optimal randomized algorithm for the ski rental problem with hierarchical prediction intervals.
In \Cref{sec:CriticalAccuracy}, building on the LP from \Cref{sec:SkiRentalLPform}, we give a method to compute the prediction accuracy required for the optimal DRCR to attain a desired value. In \Cref{sec:CriticalDelta}, using the method given in \Cref{sec:CriticalAccuracy}, we derive a necessary condition on the prediction accuracy under which, in the single-interval case, the optimal DRCR strictly improves upon the robustness value.

\subsection{LP formulation of the optimal algorithm}\label{sec:SkiRentalLPform}
We begin by presenting a method for computing an optimal randomized algorithm for the ski rental problem with hierarchical prediction intervals.
We note that, in the special case with a single prediction interval, \citet{SHCHWB24} have already provided a method for obtaining an optimal randomized algorithm for this problem.
We extend their framework to hierarchical prediction intervals and derive an optimal randomized algorithm in this more general setting.
Any randomized algorithm for this problem can be represented as the purchase-day distribution over the day $t$ on which a purchase is made.
We denote by $f_t$ the probability of purchasing the skis on day $t$. 
If the number of skiing days is $\tau$, then buying on day $t\le\tau$ incurs a cost of $(t-1)+B$ (rent through day $t-1$, then purchase on day $t$), whereas buying on day $t>\tau$ means no purchase occurs by day $\tau$, so the cost is $\tau$.
Hence, if the number of skiing days is $\tau$, the expected cost is
$$\sum_{t=1}^{\tau}(t-1+B)f_t+\tau\left(1-\sum_{t=1}^{\tau} f_t\right).$$
Our goal is to design an algorithm with a low expected cost.
To this end, we adopt the DRCR as our objective and aim to compute an algorithm that minimizes the DRCR for a given prediction profile.

Before we turn to the main results, we show that any randomized algorithm that rents skis forever with a positive probability can be excluded from consideration.
Such algorithms can be represented by setting $f_\infty>0$, where $f_\infty$ denotes the probability that the algorithm never buys the skis.
Note that $f_\infty=1-\sum_{t\in\mathcal{I}}f_t$.
Thus, $f_\infty>0$ if and only if $\sum_{t\in\mathcal{I}}f_t<1$.

If $\delta_n=0$, then the number of skiing days must be at most $u_n$.
Hence, any algorithm that might rent forever is equivalent to one that instead buys on day $u_n+1$ with the same probability.
Thus, in this case, there exists an optimal algorithm for which $f_{\infty}=0$.

Now, suppose that $\delta_n>0$ and fix any algorithm with $f_\infty>0$.
Consider an instance distribution that draws $\lceil 2B/(f_\infty\delta_n)\rceil\in\mathcal{I}$ with probability $\delta_n$; this distribution belongs to $\mathcal{D}_{\bm{\Theta},\bm{\delta}}$.
On the event that $\lceil 2B/(f_\infty\delta_n)\rceil$ is drawn, the probability that the algorithm continues renting until the end is at least $f_\infty$, and hence incurs a cost at least $\lceil 2B/(f_\infty\delta_n)\rceil$.
Thus, the DRCR of the algorithm is at least
\begin{align}
\delta_n\cdot \frac{f_\infty\cdot \lceil 2B/(f_\infty\delta_n)\rceil}{\min\{\lceil 2B/(f_\infty\delta_n)\rceil,B\}}
\ge \delta_n\cdot\frac{f_\infty\cdot 2B/(f_\infty\delta_n)}{B}
=2.
\end{align}
On the other hand, the deterministic algorithm that buys the skis on day $B$ achieves a robustness of $2-1/B$~\citep{KMRS88}.
Therefore, no optimal randomized algorithm can satisfy $f_\infty>0$.

We now introduce an infinite-dimensional LP, and then show that it can be reduced to an LP with finitely many variables and constraints.
In the infinite-dimensional LP constructed below, we treat the consistencies $c_1,c_2,\ldots,c_{n+1}$ (with $c_{n+1}$ representing robustness) as variables, together with $(f_t)_{t\in\mathcal{I}}$.
For each $\tau\in\mathcal{I}$, let $i(\tau)$ be the unique index $i\in[n+1]$ such that $\tau\in\Theta_i\setminus\Theta_{i-1}$, where $\Theta_0=\emptyset$ and $\Theta_{n+1}=\mathcal{I}$.
Recall that, if the number of skiing days is $\tau$, the expected cost of the algorithm and the offline optimal cost are $\sum_{t=1}^\tau (t-1+B)f_t+\tau(1-\sum_{t=1}^\tau f_t)$ and $\min\{\tau,B\}$, respectively.
Thus, the consistency of the algorithm for consistent set $\Theta_i$~($i\in[n+1]$) is given by
\begin{align}\label{eq:c_iinduction}
    c_i
    &=\max_{\tau\in\Theta_i}\frac{\sum_{t=1}^\tau (t-1+B)f_t+\tau(1-\sum_{t=1}^\tau f_t)}{\min\{\tau,B\}}\\
    &=\max\left\{c_{i-1},\, \max_{\tau\in\Theta_i\setminus\Theta_{i-1}}\frac{\sum_{t=1}^\tau (t-1+B)f_t+\tau(1-\sum_{t=1}^\tau f_t)}{\min\{\tau,B\}}\right\},
\end{align}
where $c_0$ is defined to be $1$.
Thus, the DRCR can be represented as $\sum_{i=1}^{n+1}(\delta_{i-1}-\delta_i)\cdot c_i$ by \Cref{thm:DRCRConvexCombination}.
Using this notation, the problem of finding an online algorithm that minimizes DRCR reduces to the following LP:
\begin{subequations}
\begin{align}
\renewcommand{\arraystretch}{1.8}
     \min \quad& \sum_{i=1}^{n+1}(\delta_{i-1}-\delta_i)\cdot c_i & \tag{\theparentequation}\label{eq:Inf-LP-SkiRental}\\
     \text{s.t.}\quad& \textstyle\sum_{t=1}^{\tau}(t-1+B)f_t+\tau\left(1-\sum_{t=1}^{\tau} f_t\right)
    \le \min\{\tau,B\}\cdot c_{i(\tau)} && (\tau\in\mathcal{I}), \label{eq:Inf-LP-SkiRental-a}\\
     & c_i\le c_{i+1}              && (i\in[n]), \label{eq:Inf-LP-SkiRental-b}\\
     & \textstyle\sum_{t\in\mathcal{I}}f_t=1, &                  \label{eq:Inf-LP-SkiRental-c}\\
     & f_t\ge 0\ (t\in S),\quad c_i\ge 0\ (i\in[n+1]). 
\end{align}
\end{subequations}
Constraint \eqref{eq:Inf-LP-SkiRental-c}, together with $f_t\ge0$ for all $t\in\mathcal{I}$, ensures that $(f_t)_{t\in\mathcal{I}}$ forms a valid probability distribution over purchase days.

We now show that the optimal solution to LP~\eqref{eq:Inf-LP-SkiRental} indeed yields an online algorithm that minimizes the DRCR.
For any randomized algorithm $(f_t)_{t\in\mathcal{I}}$ and its consistencies $(c_i)_{i\in[n+1]}$ defined in~\eqref{eq:c_iinduction}, the pair $((f_t)_{t\in\mathcal{I}},(c_i)_{i\in[n+1]})$ satisfies the constraints in \eqref{eq:Inf-LP-SkiRental-a} and \eqref{eq:Inf-LP-SkiRental-b}.
The objective is $\sum_{i=1}^{n+1}(\delta_{i-1}-\delta_i)c_i$, which equals the DRCR of the randomized algorithm by Theorem~\ref{thm:DRCRConvexCombination}.
Thus, the optimal value of LP~\eqref{eq:Inf-LP-SkiRental} is at most the minimum DRCR value. Conversely, for any feasible solution $((f_t)_{t\in\mathcal{I}},(c'_i)_{i\in[n+1]})$, the corresponding randomized algorithm $(f_t)_{t\in\mathcal{I}}$ achieves a DRCR of at most $\sum_{i=1}^{n+1}(\delta_{i-1}-\delta_i)c'_i$. 
To see this, define $(c_i)_{i\in[n+1]}$ as the consistencies induced by the algorithm.
We prove $c'_i\ge c_i$ for all $i\in[n+1]$ by induction.
Let $c'_0=c_0=1$. 
Assume that $c'_{i-1} \ge c_{i-1}$ for some $i\in[n+1]$, then
\begin{align}
c'_i
&\ge\max\left\{c'_{i-1},\,\max_{\tau\in\Theta_i\setminus\Theta_{i-1}}\frac{\sum_{t=1}^\tau (t-1+B)f_t+\tau(1-\sum_{t=1}^\tau f_t)}{\min\{\tau,B\}}\right\}\\
&\ge\max\left\{c_{i-1},\,\max_{\tau\in\Theta_i\setminus\Theta_{i-1}}\frac{\sum_{t=1}^\tau (t-1+B)f_t+\tau(1-\sum_{t=1}^\tau f_t)}{\min\{\tau,B\}}\right\}
=c_i,
\end{align}
where the first inequality follows from the constraints in \eqref{eq:Inf-LP-SkiRental-a} and \eqref{eq:Inf-LP-SkiRental-b}, the second inequality holds by the induction hypothesis, and the last equality is by \eqref{eq:c_iinduction}.
Thus, $c'_i\ge c_i$ holds for all $i\in[n+1]$, and the DRCR $\sum_{i=1}^{n+1}(\delta_{i-1}-\delta_i)c_i$ is at most $\sum_{i=1}^{n+1}(\delta_{i-1}-\delta_i)c'_i$.
Therefore, the optimal value of LP~\eqref{eq:Inf-LP-SkiRental} equals the minimum DRCR value achievable.
Moreover, for any optimal solution to the LP, the corresponding online algorithm attains the minimum DRCR.

To make LP~\eqref{eq:Inf-LP-SkiRental} computationally tractable, we exploit the structure of the problem to derive an equivalent finite-dimensional LP.
This is necessary because LP~\eqref{eq:Inf-LP-SkiRental} has infinitely many variables and constraints, making standard LP solvers inapplicable in general.
This reduction relies on the following structural property: there exists at least one optimal purchase-day distribution whose support (the set of days with positive purchase probability) can be restricted to a specific subset of days determined by the predictions $\bm\Theta$ and the purchase cost $B$, and with size is at most $B+2n+1$.
We now denote such a restricted subset of days by $S$.
Furthermore, we denote by $T$ the subset of instances that needs to be considered in the reduced LP.
Specifically, $S$ and $T$ are defined as follows: 
\begin{itemize}
\item If $B<u_n$, let $j=\min\{i\in[n]\mid u_i>B\}$. Define $S\coloneqq\{1,2,\dots,B\}\cup\{u_j+1,u_{j+1}+1,\dots,u_n+1\}$ and $T\coloneqq\{1,2,\ldots,B\}\cup\{\ell_i-1,\,u_i\mid i\in[n]\}\cup\{u_n+1\}$. 
\item If $u_n\le B$, define $S\coloneqq[B]$ and $T\coloneqq[B+1]$. 
\end{itemize}
Based on this notation, we define the finite-dimensional LP:
\begin{subequations}
\begin{align}
\renewcommand{\arraystretch}{1.8}
     \min \quad& \sum_{i=1}^{n+1}(\delta_{i-1}-\delta_i)\cdot c_i \tag{\theparentequation}\label{eq:ReducedSkiRental}\\
     \text{s.t.} \quad&\textstyle\sum_{t\in S:\, t\le\tau}(t-1+B)\cdot f_t+\tau\left(1-\sum_{t\in S:\, t\le\tau} f_t\right)\le \min\{\tau,B\}\cdot c_{i(\tau)} && (\tau\in T),  \tag{$x_\tau$}\\
     & c_i\le c_{i+1} && (i\in[n]), \tag{$y_i$}\\
     & \textstyle\sum_{t\in S}f_t=1,  \tag{$z$}\\
     & f_t\ge 0~(t\in S),\quad c_i\ge 0~(i\in[n+1]).
\end{align}
\end{subequations}
\begin{thm}\label{thm:LPpoly}
    The infinite-dimensional LP~\eqref{eq:Inf-LP-SkiRental} can be reduced to the finite-dimensional LP~\eqref{eq:ReducedSkiRental}.
\end{thm}
The proof is deferred to Appendix~\ref{subsec:LPpoly}. 
As a consequence, an optimal randomized algorithm for the ski rental problem with hierarchical prediction intervals can be obtained by solving LP~\eqref{eq:ReducedSkiRental}, which has $O(B+n)$ variables and constraints.
This LP can be solved using standard LP solvers (e.g., those based on interior-point methods~\citep{K84}).

\subsection{A polyhedral characterization of prediction accuracies}\label{sec:CriticalAccuracy}
In the previous section, we constructed an LP that computes a randomized algorithm achieving the optimal DRCR for the ski rental problem with hierarchical prediction intervals.
A natural next question is which prediction profiles lead to an improved optimal DRCR value.
As a preliminary step to answering this question, we characterize the set of prediction accuracies under which an algorithm can achieve the prescribed performance level.
Specifically, by fixing hierarchical prediction intervals $\bm\Theta=(\Theta_i)_{i\in[n]}$ and a target optimal DRCR value $v$, we characterize the set of prediction accuracies $\bm\delta=(\delta_i)_{i\in[n]}$ for which the optimal DRCR is at least $v$ via a system of linear inequalities.

We construct the system of linear inequalities based on the dual of LP~\eqref{eq:ReducedSkiRental}.
There are two reasons for adopting this dual-based approach.
First, in a minimization problem, it is difficult to directly express the condition that the optimal value of the objective function is at least $v$, whereas in the dual (maximization) problem, this condition can be expressed directly by requiring the objective value to be at least $v$.
Second, if we construct the system based on the primal problem while treating $\bm\delta$ as decision variables, the term $\sum_{i=1}^{n+1}(\delta_{i-1}-\delta_i)\cdot c_i$ contains products of variables, making the system nonlinear. In contrast, a construction based on the dual problem avoids this issue and allows the system of inequalities to remain linear.

We use the labels of the primal constraints (i.e., $(x_\tau)_{\tau\in T}$, $(y_i)_{i\in[n]}$, and $z$) as dual variables.
The dual of LP~\eqref{eq:ReducedSkiRental} is given by
\begin{subequations}
\begin{align}
\renewcommand{\arraystretch}{1.8}
     \max \quad& z+\sum_{\tau\in T}\tau\cdot x_\tau \tag{\theparentequation}\label{eq:ReducedSkiRentalDual}\\
     \text{s.t.} \quad&\textstyle z+\sum_{\tau\in T:\, \tau\ge t}\left(\tau-(t-1+B)\right)\cdot x_\tau\le0 && (t\in S),  \tag{$f_t$}\\
     &\textstyle \sum_{\tau\in T\cap(\Theta_i\setminus\Theta_{i-1})}\min\{\tau,B\}\cdot x_\tau+y_{i-1}-y_{i}
      \le\delta_{i-1}-\delta_i && (i\in[n+1]), & \tag{$c_i$} \\
     & x_\tau\ge 0~(\tau\in T),\quad y_i\ge 0~(i\in[n]).
\end{align}
    where $y_0$ and $y_{n+1}=0$ set to be $0$.
\end{subequations}

Then, we drop the objective function of LP~\eqref{eq:ReducedSkiRentalDual} and instead impose the constraint that the objective value is at least a target value (i.e., $v\le z+\sum_{\tau\in T}\tau\cdot x_\tau$), thereby obtaining a system of linear inequalities.
Under this operation, $z$ becomes effectively an auxiliary variable, and we can eliminate it from the system by combining constraints $(f_t)$.
In addition, we treat $(\delta_i)_{i\in[n]}$ as decision variables and append the constraints $1\ge\delta_1\ge\delta_2\ge\cdots\ge\delta_n\ge0$, which are implied by the hierarchical prediction structure. 
This yields a system of linear inequalities that characterizes the feasible region of prediction accuracies $(\delta_i)_{i\in[n]}$ for which the optimum DRCR is at least $v$.
Formally, the system is
\begin{align}
\renewcommand{\arraystretch}{1.5}
\left\{
\begin{array}{ll}\label{eq:SysLinIneq}
     v\le\sum_{\tau\in T:\,\tau<t}\tau x_\tau-(t-1+B)\sum_{\tau\in T:\, \tau\ge t}x_\tau & (t\in S), \\
     \sum_{\tau\in T\cap(\Theta_i\setminus\Theta_{i-1})}\min\{\tau,B\}\cdot x_\tau+y_{i-1}-y_{i}
      \le\delta_{i-1}-\delta_i & (i\in[n+1]),\\
     1=\delta_0\ge\delta_1\ge\delta_2\ge\cdots\ge\delta_n\ge \delta_{n+1}=0,\\
     x_\tau\ge0\ (\tau\in T),\quad y_i\ge0\ (i\in[n]),\quad y_0=y_{n+1}=0.
\end{array}
\right.
\end{align}
Then, we obtain the following theorem.

\begin{thm}\label{thm:accuracies}
Fix hierarchical prediction intervals $\bm{\Theta}=(\Theta_i)_{i\in[n]}$ and a target DRCR value $v$.
Then, the set of prediction accuracies $\bm{\delta}=(\delta_i)_{i\in[n]}$ for which the optimal DRCR under the prediction profile $(\bm{\Theta},\bm{\delta})$ is at least $v$ is characterized by the linear system \eqref{eq:SysLinIneq}.
\end{thm}

\subsection{Applications of the structure and characterization}\label{sec:CriticalDelta}
We integrate the results obtained in the previous sections to derive conditions that characterize a certain class of desirable prediction profiles for the ski rental problem.
For simplicity, we focus on the single prediction interval setting, where the last day of skiing $\tau$ belongs to the interval $[\ell,u]$ with probability at least $1-\delta$, and we write $\bm{\Theta}=\{\Theta=[\ell,u]\}$, $\bm{\delta}=\{\delta\}$.

To begin, we experimentally observe how the optimal DRCR value varies with prediction accuracy.
Figure~\ref{fig:optDRCR} shows the relationship between $\delta$ and the optimal DRCR for two problem instances.
These plots can be drawn by solving LP~\eqref{eq:ReducedSkiRental} to obtain the optimal DRCR value for each $\delta$.
These two instances are chosen as representative examples, and similar qualitative behavior can be observed for other choices of $B$, $\ell$, and $u$.
From the figure, we confirm that the optimal DRCR has monotonicity (non-decreasing) and concavity in $\delta$, consistent with the results in \Cref{sec:DRCRstructures}.
Moreover, we can observe a threshold value of $\delta$ at which the optimal DRCR value stops increasing and coincides with the optimal robustness value.
This shows that there exist prediction profiles that do not improve the optimal DRCR even when a positive accuracy guarantee is provided, and it naturally leads to the question of which prediction profiles allow the optimal DRCR to be strictly better than the robustness value.

\begin{figure}[ht]
\centering
\begin{minipage}[b]{0.49\columnwidth}
    \centering
    \includegraphics[width=0.9\columnwidth]{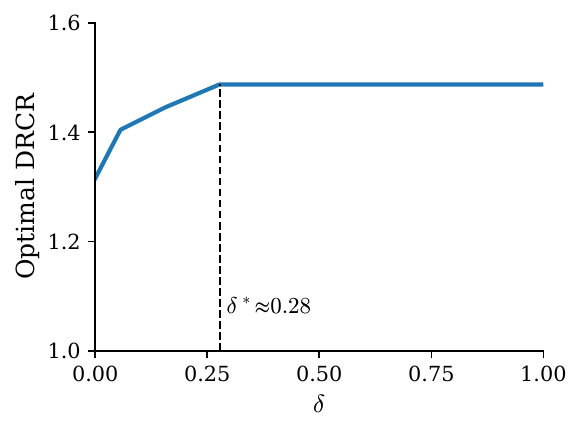}
    \subcaption{$\ell=3$, $u=8$, and $B=5$}
    \label{fig:a}
\end{minipage}
\begin{minipage}[b]{0.49\columnwidth}
    \centering
    \includegraphics[width=0.9\columnwidth]{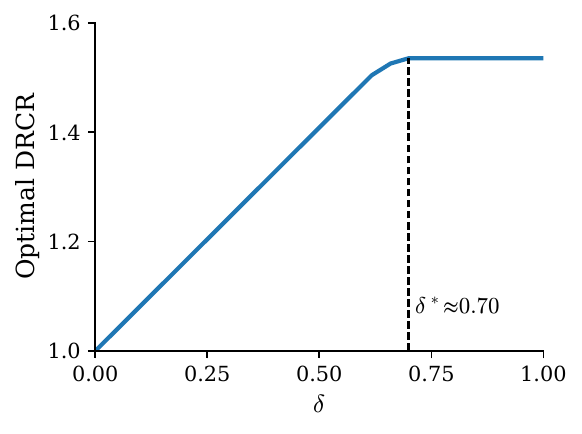}
    \subcaption{$\ell=3$, $u=8$, and $B=10$}
    \label{fig:b}
\end{minipage}
\caption{Examples of the relationship between the optimal DRCR value and $\delta$.}
\label{fig:optDRCR}
\end{figure}

To answer this question, we propose a method to compute, for given $\ell$, $u$, and $B$, the minimum prediction accuracy required for the optimal DRCR to be strictly smaller than the optimal robustness value.
We refer to this threshold as the \emph{critical accuracy} and denote it by $\delta^*$ in what follows.
Formally, with $\ell$, $u$, and $B$ fixed, the critical accuracy $\delta^*$ is defined as the smallest value of $\delta$ for which the optimal DRCR value coincides with the optimal robustness value.
Our method is based on System~\eqref{eq:SysLinIneq}, which characterizes, for fixed prediction set $\bm{\Theta}$ and target value $v$, the set of prediction accuracies $\bm{\delta}$ under which the optimal DRCR is at least $v$.
To compute $\delta^*$, we set the target value $v$ to the optimal robustness value and minimize $\delta$ subject to System~\eqref{eq:SysLinIneq}.
To specify the value of $v$, we observe that the definitions of robustness and of the competitive ratio are the same.
Thus, the optimal robustness value equals the competitive ratio achieved by the optimal randomized algorithm for the classical ski rental problem,
which is given by \citet{K98} as
$$\frac{B^B}{B^B-(B-1)^B}.$$
Summarizing the above discussion, the LP for computing the critical accuracy is given as follows:
\begin{subequations}
\begin{align}
\renewcommand{\arraystretch}{1.8}
     \min \quad& \delta \tag{\theparentequation}\label{eq:LPforCriticalDelta}\\
     \text{s.t.} \quad&\textstyle\frac{B^B}{B^B-(B-1)^B}\le \sum_{\tau\in T:\, \tau<t}\tau x_\tau+(t-1+B)\sum_{\tau\in T:\, \tau\ge t} x_\tau \quad (t\in S),  \\
     &\textstyle \sum_{\tau\in T\cap\Theta}\min\{\tau,B\}\cdot x_\tau-y
      \le1-\delta, \\
     &\textstyle \sum_{\tau\in T\cap(\mathcal{I}\setminus\Theta)}\min\{\tau,B\}\cdot x_\tau+y
      \le\delta, \\
     & x_\tau\ge 0~(\tau\in T),\quad y\ge 0, \quad0\le\delta\le1.
\end{align}
\end{subequations}
Here, we can choose $S=\{1,2,\dots,B,u+1\}$ and $T=\{1,2,\dots,B\}\cup\{\ell-1,u,u+1\}$.
Therefore, this LP has polynomial size and can be solved in polynomial time.

\section{Conclusion}
In this paper, we analyze several structural properties of the relationship between the DRCR and prediction accuracy, and further consider an extension to hierarchical predictions.
As an application of the derived structural properties, we consider the ski rental problem with a prediction interval and derive an LP that provides conditions under which the optimal DRCR is strictly smaller than the optimal robustness value.

As future work, it would be interesting to explore whether our approach for computing $\delta^*$ extends to other online optimization problems whose optimal algorithms are computable through linear programming formulations.
Moreover, a more detailed investigation is needed to understand how predictions should be constructed from given data in order to effectively minimize the DRCR.

Also, practical limitations of hierarchical predictions should be considered.
While hierarchical predictions offer a theoretically advantageous structure, it remains uncertain whether real-world predictors, especially learning models, can produce such predictions.
Therefore, it is important to carefully examine the validity of the hierarchical assumption and any associated limitations.

\bibliographystyle{plainnat} 
\bibliography{references}

\newpage
\appendix

\section{Deferred proofs}

\subsection{Proof of Theorem~\ref{thm:LPpoly}}\label{subsec:LPpoly}
\begin{proof}
To eliminate redundant variables and constraints from LP~\eqref{eq:Inf-LP-SkiRental}, we first introduce a transformation on feasible solutions.
The operation maps a feasible solution $(\bm{f},\bm{c})=((f_t)_{t\in\mathcal{I}},(c_i)_{i\in[n+1]})$ to another feasible solution $(\bm{\hat{f}},\bm{c})=((\hat{f}_t)_{t\in\mathcal{I}},(c_i)_{i\in[n+1]})$.
Since the consistencies $\bm{c}$ remain unchanged, this operation preserves the objective value.
Let $C_{\tau}$ denote the constraint indexed by $\tau\in\mathcal{I}$ in \eqref{eq:Inf-LP-SkiRental-a}.
For a feasible solution $(\bm{f},\bm{c})$, we denote $\mathrm{LHS}(\bm{f},\tau)$ and $\mathrm{RHS}(\bm{c},\tau)$ as the values of the left-hand side and right-hand side of $C_\tau$, respectively.

We observe that, if there exists a day $\tau^*\in\mathcal{I}$ such that 
\begin{align}
\mathrm{RHS}(\bm{c},\tau^*)\ge\mathrm{RHS}(\bm{c},\tau^*+1),\label{eq:tau*}
\end{align}
then it is possible to construct another feasible solution in which the probability of purchasing on day $\tau^*+1$ is zero, and the objective value does not worsen. Specifically, we consider a solution $(\bm{\hat{f}},\bm{c})$ where
\begin{align}
\hat{f}_t&=
\begin{cases}
f_{\tau^*}+f_{\tau^*+1} & \text{if }t=\tau^*,\\
0                     & \text{if }t=\tau^*+1,\\
f_t               & \text{otherwise}
\end{cases}
&&(t\in\mathcal{I}).\label{eq:probmassmove}
\end{align}
We now verify the feasibility of $(\bm{\hat{f}},\bm{c})$.
The feasibility of constraints other than those in \eqref{eq:Inf-LP-SkiRental-a} is straightforward.
Therefore, it suffices to verify the feasibility only for the constraints in \eqref{eq:Inf-LP-SkiRental-a}.
This verification is carried out by considering the following three cases.
\begin{itemize}
    \item For $\tau\le \tau^*-1$, we have $\hat{f}_t=f_t$ for all $t\le \tau$. Hence, we have
    $$\mathrm{LHS}(\bm{\hat{f}},\tau)=\mathrm{LHS}(\bm{f},\tau)\le \mathrm{RHS}(\bm{c},\tau).$$

    \item For $\tau\ge\tau^*+1$, we have
    \begin{align}
    \mathrm{LHS}(\bm{\hat{f}},\tau)-\mathrm{LHS}(\bm{f_{t}},\tau)
    &=(\tau^*-1+B)(f_{\tau^*}+f_{\tau^*+1})-(\tau^*-1+B)f_{\tau^*}-(\tau^*+B)f_{\tau^*+1}\\
    &=-f_{\tau^*+1}\le 0. 
    \end{align}
    Consequently, we obtain
    $$\mathrm{LHS}(\bm{\hat{f}},\tau)\le \mathrm{LHS}(\bm{f},\tau)\le \mathrm{RHS}(\bm{c},\tau).$$

    \item For $\tau=\tau^*$, we have
    \begin{align}
        \mathrm{LHS}(\bm{\hat{f}},\tau)-\mathrm{LHS}(\bm{\hat{f}},\tau+1)
        &=\tau^*\sum_{t\ge\tau^*+2} f_t-(\tau^*+1)\sum_{t\ge\tau^*+2} f_t
        =-\sum_{t\ge\tau^*+2} f_t\le 0.
    \end{align}
    Hence, we obtain
    \begin{align}
        \mathrm{LHS}(\bm{\hat{f}},\tau)
        &\le \mathrm{LHS}(\bm{\hat{f}},\tau+1)
        \le \mathrm{RHS}(\bm{c},\tau+1)
        \le \mathrm{RHS}(\bm{c},\tau),
    \end{align}
    where the second inequality follows from the discussion for the case $\tau\ge \tau^*+1$, and the last inequality holds by the assumption in \eqref{eq:tau*}.
\end{itemize}
Therefore, the new solution $(\bm{\hat{f}},\bm{c})$ is feasible, 
since the constraints in \eqref{eq:Inf-LP-SkiRental-a} are satisfied as $\mathrm{LHS}(\bm{\hat{f}},\tau)\le \mathrm{RHS}(\bm{c},\tau)$ in every case.

Using this transformation, we eliminate unnecessary variables and constraints from LP~\eqref{eq:Inf-LP-SkiRental}. 
We proceed by analyzing two cases based on the relationship between $B$ and $u_n$: (i) $B<u_n$ and (ii) $u_n\le B$.
For notational simplicity, we set $u_{n+1}=\infty$.

\medskip
\noindent\textbf{Case 1: $B<u_n$.} 
We first eliminate redundant variables from LP~\eqref{eq:Inf-LP-SkiRental}. 
To this end, we establish the following statement about feasible solutions.
Let $j=\min\{i\in[n]\mid u_i>B\}$. 
We claim that, for any feasible solution $(\bm{f},\bm{c})$, there exists another feasible solution $(\bm{\hat{f}},\bm{c})$ with the same objective value such that $\hat{f}_t\ne 0$ only for $t$ in 
$$S\coloneqq\{1,2,\dots,B\}\cup\{u_j+1,u_{j+1}+1,\dots,u_n+1\}.$$
We prove this statement by using the transformation on feasible solutions described above.
For each $\tau\ge B$, we have $\mathrm{RHS}(\bm{c},\tau)=B\cdot c_{i(\tau)}$.
By the hierarchical structure and $B<u_j$, we have $i(B)\ge i(B+1)\ge\dots\ge i(u_j)$. Since $1\le c_1\le\cdots\le c_{n+1}$, it follows that $\mathrm{RHS}(\bm{c},B)\ge \mathrm{RHS}(\bm{c},B+1)\ge\cdots\ge\mathrm{RHS}(\bm{c},u_j)$.
By repeatedly applying transformation~\eqref{eq:probmassmove} for $\tau^*$ from $u_j-1$ to $B$, we can shift the entire probability mass from the interval $\{B+1,B+2,\dots,u_j\}$ onto day $B$ without violating feasibility. 
Similarly, for each $i\in\{j,j+1,\ldots,n\}$, we can move entire probability mass from the interval $\{u_{i}+1,u_{i}+2,\ldots,u_{i+1}\}$ onto day $u_{i}+1$, since we have $\mathrm{RHS}(\bm{c},\tau)=B\cdot c_{i+1}$ for all $\tau\in\{u_i+1,u_i+2,\dots,u_{i+1}\}$.
Thus, we can construct a new feasible solution $(\bm{c},\bm{\hat{f}})$ that achieves the same objective value, where $\bm{\hat{f}}=(\hat{f}_t)_{t\in\mathcal{I}}$ is defined as follows:
\begin{align}
    \hat{f}_t=\begin{cases}
    f_t &\text{if } t\le B-1,\\
    \sum_{k=B}^{u_j} f_k &\text{if } t=B,\\
    \sum_{k=u_i+1}^{u_{i+1}} f_k &\text{if } t=u_i+1\le u_{i+1}\text{ for some }i\in\{j,j+1,\ldots,n\},\\
    0 & \text{otherwise}
    \end{cases}
&&(t\in\mathcal{I}).
\end{align}
This transformation proves that an optimal solution can be found within the subspace where $f_t=0$ for all $t\notin S$.
Therefore, we can eliminate the variables $f_t$ for $t\notin S$ from LP~\eqref{eq:Inf-LP-SkiRental} without loss of optimality.

Next, using the LP with these variables eliminated, we remove redundant constraints from \eqref{eq:Inf-LP-SkiRental-a}.
Define 
$$T\coloneqq\{1,\ldots,B\}\cup\{\ell_i-1,\,u_i\mid i\in[n]\}\cup\{u_n+1\}.$$
We show that all constraints $C_\tau$ with $\tau\notin T$ are redundant.
Let $(\bm{c}, (\hat{f}_t)_{t\in S})$ be a solution to the reduced LP.
To begin with, since $\hat{f}_t=0$ for all $t\ge u_n+2$, all constraints $(C_\tau)_{\tau\ge u_n+1}$ take the same form $\sum_{t\in S}(t-1+B)f_t+\tau(1-\sum_{t\in S}f_t)\le B\cdot c_{n+1}$.
Thus, it suffices to verify the constraint at $\tau=u_n+1$ instead of verifying all constraints $(C_\tau)_{\tau\ge u_n+2}$.
Then, for each $i\in\{j,j+1,\ldots,n\}$, consider the set $(\Theta_i\setminus \Theta_{i-1})\cap\{B,B+1,\ldots\}$.
For every $\tau$ in this set, we have $\mathrm{RHS}(\bm{c},\tau)=B\cdot c_{i(\tau)}$, and $\mathrm{LHS}\big(\bm{\hat{f}},\tau\big)$ is non-decreasing in $\tau$.
Therefore, if there exists $\tau^*\in(\Theta_i\setminus\Theta_{i-1})\cap\{B,B+1,\ldots\}$ such that $\mathrm{LHS}(\bm{\hat{f}},\tau^*)\le\mathrm{RHS}(\bm{c},\tau^*)$ holds, then $\mathrm{LHS}(\bm{\hat{f}},\tau)\le \mathrm{RHS}(\bm{c},\tau)$ automatically holds for every $\tau\in (\Theta_i\setminus\Theta_{i-1})\cap\{B,B+1,\ldots\}$ with $\tau\le\tau^*$.
Therefore, for each $i$, it suffices to verify that the constraint is satisfied at the largest element in $(\Theta_i\setminus \Theta_{i-1})\cap\{B,B+1,\ldots\}$, and such element $\tau^*$ is one of the following:
$$
\tau^*=u_i \ \text{if } u_i>u_{i-1},\qquad
\tau^*=\ell_{i-1}-1 \ \text{if } u_i=u_{i-1}\ \text{and}\ \ell_i<\ell_{i-1}.
$$
Consequently, checking the constraints in $(C_\tau)_{\tau\in T}$ is sufficient. 

\medskip
\noindent\textbf{Case 2: $u_n\le B$.} As in Case 1, we first eliminate redundant variables by the transformation on the solutions, and then remove redundant constraints.
We first show that, for any feasible solution $(\bm{f},\bm{c})$, there exists another feasible solution $(\bm{\hat{f}},\bm{c})$ with the same objective value such that $\hat{f}_t\ne 0$ only for $t$ in 
$$S\coloneqq [B].$$
For all $\tau\ge B$, we have $\mathrm{RHS}(\tau)=c_{n+1}B$.
By repeatedly applying transformation~\eqref{eq:probmassmove} for $\tau^*$ from $\infty$ to $B$, we can shift the entire probability mass from all day $t\ge B+1$ onto day $B$ without violating feasibility. Consequently, we can construct a new feasible solution $(\bm{\hat{f}},\bm{c})$ that achieves the same objective value, where $\bm{\hat{f}}$ is defined as follows:
\begin{align}
    \hat{f}_t=\begin{cases}
        f_t &\text{if } t\le B-1,\\
        \sum_{k=B}^{\infty} f_k~(=1-\sum_{k=1}^{B-1}f_k) &\text{if } t=B,\\
        0 & \text{if }t>B
    \end{cases}
&&(t\in\mathcal{I}).
\end{align}
By the same discussion as above, we can eliminate the variables $f_t$ for $t\notin S$ from LP~\eqref{eq:Inf-LP-SkiRental} without loss of optimality.

Next, using the LP with these variables eliminated, we remove redundant constraints from \eqref{eq:Inf-LP-SkiRental-a}.
We show that all constraints $C_\tau$ with $\tau\not\in T$ are redundant, where
$$T\coloneqq [B+1].$$
Since $\hat{f}_t=0$ for all $t\ge B+1$, it suffices to verify the constraint $C_{B+1}$ to ensure satisfaction for all constraints $(C_\tau)_{\tau\ge B+1}$. 
Consequently, checking the constraints $(C_\tau)_{\tau\in T}$ is sufficient.

\medskip

Combining these cases, we conclude that the infinite-dimensional LP~\eqref{eq:Inf-LP-SkiRental} can be reduced to the finite-dimensional LP~\eqref{eq:ReducedSkiRental}.
\end{proof}

\end{document}